\documentclass[9.5bp,journal,compsoc]{IEEEtran}
\usepackage{graphicx}
\usepackage{epsfig}
\usepackage{epstopdf} 
\usepackage[pagebackref=false,breaklinks=true,colorlinks,linkcolor={black},citecolor={black},urlcolor={blue},bookmarks=false]{hyperref}
\usepackage{cite}
\usepackage{microtype}
\usepackage{subcaption}
\usepackage{url}

\usepackage{array}
\usepackage[normalem]{ulem}
\usepackage{ulem}
\usepackage{amsmath}
\usepackage{longtable}
\usepackage{rotating}
\usepackage{multirow}
\usepackage{makecell}
\usepackage{bm}
\usepackage{float}
\usepackage{flushend}
\usepackage{xcolor}
\usepackage{soul}
\usepackage{booktabs}
\usepackage{threeparttable}
\usepackage{color}
\usepackage{algorithmic}
\usepackage{tabularx}
\usepackage{amsfonts}
\usepackage{epstopdf}
\usepackage{latexsym}
\usepackage{amssymb}
\usepackage{footnote}
\usepackage{pifont}
\usepackage[ruled]{algorithm2e}

\newtheorem{theorem}{Theorem}

\newtheorem{proof}{Proof}

\usepackage{threeparttable}

\begin{document}
	\title{Self-Supervised Graph Embedding Clustering}
	\author{Fangfang~Li, Quanxue~Gao,  Cheng Deng, Wei Xia
	}
	
	\markboth{}%
	{Shell \MakeLowercase{\textit{Xia et al.}}}
	
	\IEEEtitleabstractindextext{%
		
\begin{abstract}
	The K-means one-step dimensionality reduction clustering method has made some progress in addressing the curse of dimensionality in clustering tasks. However, it combines the K-means clustering and dimensionality reduction processes for optimization, leading to limitations in the clustering effect due to the introduced hyperparameters and the initialization of clustering centers. Moreover, maintaining class balance during clustering remains challenging.
	To overcome these issues, we propose a unified framework that integrates manifold learning with K-means, resulting in the self-supervised graph embedding framework. Specifically, we establish a connection between K-means and the manifold structure, allowing us to perform K-means without explicitly defining centroids. Additionally, we use this centroid-free K-means to generate labels in low-dimensional space and subsequently utilize the label information to determine the similarity between samples. This approach ensures consistency between the manifold structure and the labels.
	Our model effectively achieves one-step clustering without the need for redundant balancing hyperparameters. Notably, we have discovered that maximizing the $\ell_{2,1}$-norm naturally maintains class balance during clustering, a result that we have theoretically proven. Finally, experiments on multiple datasets demonstrate that the clustering results of Our-LPP and Our-MFA exhibit excellent and reliable performance.
	
\end{abstract}

		}
	
	\maketitle
	
	\IEEEdisplaynontitleabstractindextext
	\IEEEpeerreviewmaketitle
	
	\IEEEraisesectionheading{\section{Introduction}\label{sec:introduction}}
The objective of clustering is to partition data into multiple clusters, ensuring that samples within the same cluster exhibit high degrees of similarity while those in different clusters differ significantly. However, in high-dimensional spaces, clustering faces the "curse of dimensionality." Due to the sparse distances between samples, traditional algorithms struggle to identify similarities. Additionally, noise and redundant information in high-dimensional data complicate the clustering process, further hindering the accuracy of the results \cite{MeiZGYG23}. 

In response, dimensionality reduction algorithms have emerged. These algorithms fall into two main categories: (1) those based on subspaces, such as principal component analysis (PCA) \cite{6793549} and linear discriminant analysis (LDA) \cite{BelhumeurHK97}; and (2) those based on manifold structures, including locality preserving projections (LPP) \cite{HeN03} and Marginal Fisher Analysis (MFA) \cite{YanXZZ05}.

In recent years, research on dimensionality reduction clustering has received widespread attention. The most direct approach is typically to perform clustering operations after applying dimensionality reduction preprocessing to the data. For instance, a common method entails utilizing PCA to derive the principal components of the data, followed by clustering these principal components (i.e., PCAKM). To reduce the effects of feature redundancy and noise, Fu et al. \cite{Fu0CZW21} learned subspace representations at both the feature and data levels, while Liu et al. \cite{LiuLZG23} projected the data into a low-dimensional space before performing subspace representations. Similarly, spectral clustering can be viewed as a specialized instance of this integrated approach, which transforms the data into spectral space through Laplacian decomposition and then performs clustering. However, the segregation of dimensionality reduction and clustering may lead to extracted features that are unsuitable for the clustering task, resulting in subpar clustering performance.

Some one-step dimensionality reduction clustering methods have been developed to address these challenges. Ding et al. \cite{DingL07} alternated LDA and K-means processes, performing K-means in the embedding space to obtain labels to guide LDA optimization. However, this alternating strategy may cause non-convergence problems. Wang et al. \cite{WangWNLYW21} and Wang et al. \cite{WangWRLN23} proved the equivalence between LDA and K-means, rewriting LDA as an unsupervised LDA model to achieve clustering and low-dimensional embedding synchronization. Additionally, Nie et al. \cite{NieDHWL23} revealed the equivalence of discriminative dimensionality reduction clustering and LDA using the regression form of LDA \cite{NieXLHZ12}. Hou et al. \cite{HouNYT15} proposed a unified one-step dimensionality reduction clustering framework, which preserves the main statistical information of the data through PCA and applies K-means clustering in the subspace, optimizing them alternately to ensure reliable clustering results while reducing the complexity of the data. Wang et al. \cite{WangCZHY19} unified LPP and K-means to achieve dimensionality reduction clustering based on local structure through a balanced hyperparameter.


Although the one-step dimensionality reduction clustering methods mentioned earlier have made significant progress, they only jointly optimize K-means and graph embedding through a balancing hyperparameter, meaning that the manifold structure and clustering labels are independent. Additionally, these methods rely on K-means, making the algorithm highly sensitive to the selection of clustering centroids and challenging to ensure class balance.

To address these limitations, we propose a one-step self-supervised low-dimensional manifold clustering framework that integrates centroid-free K-means and low-dimensional graph embedding into a unified model. Specifically, this approach uses centroid-free K-means to generate labels and employs this label information to construct the manifold structure, thereby maintaining label consistency within the same manifold. 

Moreover, unlike the commonly used approach of minimizing the $\ell_{2,1}$-norm for feature selection, we introduce the maximization of the $\ell_{2,1}$-norm to ensure class balance.

Our contributions are summarized as follows:
\begin{itemize}
\item [(1)] We construct a self-supervised low-dimensional manifold clustering framework, where dimensionality reduction and clustering are performed simultaneously, with discrete clustering labels guiding each other. 
\item [(2)] The manifold structure between data points is determined by both label information and local structure, ensuring consistency between the manifold structure and labels. 
\item [(3)] Centroid-free K-means is used to generate discrete labels in the low-dimensional space, enhancing the stability of the clustering algorithm. 
\item [(4)] We observed an interesting phenomenon: natural class balance can be achieved by maximizing the $\ell_{2,1}$-norm during clustering, and we provide theoretical proof for this observation.
\end{itemize}

In this paper, we use bold uppercase letters to denote matrices (e.g., $\mathbf{G}$) and bold lowercase letters to denote vectors (e.g., $\mathbf{g}$). For a given matrix $\mathbf{G}$, the symbol $\mathbf{g}_i$ denotes the $i^{th}$ column, and $\mathbf{g}^i$ denotes the $i^{th}$ row. The term $g_{ij}$ refers to the element located at the intersection of the $i^{th}$ row and the $j^{th}$ column.

\section{Related Work}\label{Related work}

\subsection{K-means Clustering}
Denote the sample matrix as $\mathbf{X}=[\mathbf{x}^1, \mathbf{x}^2,\ldots, \mathbf{x}^d]\in {\mathbb{R}^{N \times d}}$, where $N$ is the number of samples and $d$ is the feature dimension. The label matrix is $\mathbf{G}=[\mathbf{g}_1, \mathbf{g}_2,\ldots, \mathbf{g}_N]^T\in {\mathbb{R}^{N \times k}}$.
K-means aims to divide the samples into $K$ clusters by iteratively determining the center of mass and minimizing the distance between the samples and the class centroids. After initializing the class centroids, K-means can be expressed as:
\begin{equation}\label{FK}
	\begin{aligned}
		\mathop {\min }\limits_{\mathbf{G},{{\mathbf{u}}^j}} \sum\limits_{i,j} {{g_{ij}}\left\| {{{\mathbf{x}}^i} - {{\mathbf{u}}^j}} \right\|_2^2} \quad
		\text{s.t.} \ \mathbf{G} \in \text{Ind}
	\end{aligned}
\end{equation}
where $\mathbf{U}=[\mathbf{u}^1, \mathbf{u}^2,\ldots, \mathbf{u}^k]\in {\mathbb{R}^{k \times d}}$ is the centroid matrix. The discrete nature of the labels implies that if $\mathbf{x}^i$ belongs to the $j^{th}$ cluster, then $g_{ij} = 1$; otherwise, $g_{ij} = 0$.

\subsection{Graph Embedding Framework}
For dimensionality reduction tasks, LDA \cite{BelhumeurHK97}, PCA \cite{6793549}, LPP \cite{HeN03}, and MFA \cite{YanXZZ05} have different motivations, but they all aim to effectively reduce data dimensions while maximizing the retention of core information and structural features in the original data. From this unified perspective, all of the above dimensionality reduction methods can be represented within a unified graph embedding framework \cite{YanXZZ05}:
\begin{equation}
	\begin{aligned}
		&\mathop {\min }\limits_{{\mathbf{W}}} \frac{{\text{tr}({{\mathbf{W}}^T}{{\mathbf{X}}^T}{{\mathbf{LX}}}{\mathbf{W}})}}{{\text{tr}({{\mathbf{W}}^T}{{\mathbf{X}}^T}{{\mathbf{BX}}}{\mathbf{W}})}} \\
		&\quad
		\Rightarrow\mathop {\min }\limits_{\mathbf{W}} \sum\limits_{i = 1}^n {\sum\limits_{j = 1}^n {\left\| {{{\mathbf{x}}^i}{\mathbf{W}} - {{\mathbf{x}}^j}{\mathbf{W}}} \right\|_2^2{s_{ij}}} } \\
		&\text{s.t.} \ {{\mathbf{W}}^T}{{\mathbf{X}}^T}{\mathbf{BXW}} = {\mathbf{c}}
	\end{aligned}
\end{equation}
where $\mathbf{W}\in {\mathbb{R}^{d \times m}}$ is the projection matrix, and $m$ is the dimension after dimensionality reduction. The symmetric matrix $\mathbf{S}\in {\mathbb{R}^{N \times N}}$ represents the similarity matrix. $\mathbf{L}=\mathbf{D}_S-\mathbf{S}$ is the Laplacian matrix, where $\mathbf{D}_{S(ii)}=\sum\limits_{i \neq j}S_{ij}, \forall i$. $\mathbf{c}$ is a constant, and $\mathbf{B}$ is a constraint matrix.

The difference between these algorithms lies in calculating the similarity matrix $\mathbf{S}$ of the graph and selecting the constraint matrix $\mathbf{B}$. 
For LDA, $s_{ij}=\frac{{\delta}_{k_i,k_j}}{n_{k_i}}$, and $\mathbf{B}=\mathbf{I}-\frac{1}{N}ee^T$, where $e=[1,1,\ldots,1]^T\in {\mathbb{R}^{N \times 1}}$.
For PCA, $s_{ij}=\frac{1}{N}$ for $i\neq j$, and $\mathbf{B}=\mathbf{I}$.
For LPP, $s_{ij}={e^{ - \frac{{\left\| {{{\mathbf{x}}_i} - {{\mathbf{x}}_j}} \right\|_2^2}}{\sigma }}}$, and $\mathbf{B}=\mathbf{D}_\mathbf{S}=diag(\sum\limits_{j = 1}^n{s}_{1j}, \ldots,\sum\limits_{j = 1}^n{s}_{Nj})$.
And for MFA, $\mathbf{B}=\mathbf{D}_b-\mathbf{S}_b$, with
						$\begin{aligned}
	{s_{{{ij}}}}=
	\begin{cases}
		1 , & \text{if } \{{{\mathbf{x}}_i} \in \mathcal{N}_k({{\mathbf{x}}_j}) \text{ or } {{\mathbf{x}}_j} \in \mathcal{N}_k({{\mathbf{x}}_i})\} \text{ and } {\mathbf{g}}_i = {\mathbf{g}}_j;\\
		0, & \text{otherwise}.
	\end{cases}
\end{aligned}$ \\
$\begin{aligned}
	{s_{{b_{ij}}}}=
	\begin{cases}
		1 , & \text{if } \{{{\mathbf{x}}_i} \in \mathcal{N}_k({{\mathbf{x}}_j}) \text{ or } {{\mathbf{x}}_j} \in \mathcal{N}_k({{\mathbf{x}}_i})\} \text{ and } {\mathbf{g}}_i  \ne  {\mathbf{g}}_j;\\
		0, & \text{otherwise}.
	\end{cases}
\end{aligned}$

\section{Methodology}\label{method}
\subsection{Reconsidering K-means as Manifold Learning}

Most previous one-step dimensionality reduction clustering methods rely on the joint optimization of dimensionality reduction techniques and K-means. However, these methods usually treat the manifold structure \textbf{S} and clustering label \textbf{G} as independent, ignoring the potential interactions and dependencies between them. In fact, K-means can be considered a manifold learning method in the original space. Therefore, we introduce Theorem \ref{trm1} to re-express K-means from the perspective of manifold learning.

\begin{theorem}\label{trm1}
	Let $\bf{P}$ be a diagonal matrix with diagonal elements defined as ${p_{ii}} = \sum\nolimits_{j = 1}^n {{g_{ji}}}$, and let $\bf{Q}$ be a diagonal matrix with diagonal elements defined as ${q_{ii}} = \sum\nolimits_{j = 1}^c {{g_{ij}}}$. Then,
	\begin{equation}\label{rk}
		\begin{aligned}
			&\mathop {\min }\limits_{\mathbf{G},{{\mathbf{u}}^j}} \sum\limits_{i,j} {{g_{ij}}\left\| {{{\mathbf{x}}^i} - {{\mathbf{u}}^j}} \right\|_2^2} 
			= \mathop {\min }\limits_{{\mathbf{G}}} \sum\limits_{i = 1}^n {\sum\limits_{j = 1}^n {\left\| {{{\mathbf{x}}^i} - {{\mathbf{x}}^j}} \right\|_2^2s_{ij}}} \hfill \\
			&\quad \text{s.t. } {\mathbf{G}} \in Ind
		\end{aligned}
	\end{equation}
\end{theorem}
where ${\mathbf{S}} = {\mathbf{Z}}{{\mathbf{Z}}^T}$, and ${\mathbf{Z}} = {\mathbf{G}}{{\mathbf{P}}^{ - 1/2}}$.

\begin{proof}
	For the left side of Eq.~(\ref{rk}), the optimal solution for $\mathbf{u}_k$ is:
	\begin{equation}\label{mk}
		\begin{aligned}
			{{\mathbf{u}}^j} = \frac{{\sum\nolimits_i {{{\mathbf{x}}^i}{g_{ij}}} }}{{{p_{jj}}}} = {({p_{jj}})^{ - 1}}{\mathbf{g}}_j^T{\mathbf{X}}
		\end{aligned}
	\end{equation}
	Substituting Eq.~(\ref{mk}) into the left side of Eq.~(\ref{rk}), the following derivation can be made:
	\begin{equation}\label{22}
		\begin{aligned}
			&\mathop {\min }\limits_{\mathbf{G}} \sum\limits_{i = 1}^n {\text{tr}({q_{ii}}{{\mathbf{x}}^i}{{\mathbf{x}}^i}^T)}  - \sum\limits_{j = 1}^c {\text{tr}({{\mathbf{X}}^T}{\mathbf{g}_j}{{({p_{jj}})}^{ - 1}}{\mathbf{g}_j^T}{\mathbf{X}})}  \hfill \\
			&\quad\quad = \mathop {\min }\limits_{\mathbf{G}} \text{tr}({{\mathbf{X}}^T}{\mathbf{QX}}) - \text{tr}({{\mathbf{X}}^T}{\mathbf{G}}{{\mathbf{P}}^{ - 1}}{{\mathbf{G}}^T}{\mathbf{X}}) \hfill \\
			&\quad\quad = \mathop {\min }\limits_{\mathbf{G}} \text{tr}({{\mathbf{X}}^T}({\mathbf{Q}} - {\mathbf{G}}{{\mathbf{P}}^{ - 1}}{{\mathbf{G}}^T}){\mathbf{X}}) \hfill \\ 
		\end{aligned}
	\end{equation}
	
	Additionally, for Eq.~(\ref{rk}), when the adjacency matrix ${\mathbf{S}} = {\mathbf{Z}}{{\mathbf{Z}}^T} = {\mathbf{G}}{{\mathbf{P}}^{ - 1}}{\mathbf{G}}^T$, then
	\begin{equation}\label{}
		\begin{aligned}
			{\mathbf{S1}} = {\mathbf{G}}{{\mathbf{P}}^{ - 1}}{{\mathbf{G}}^T}{\mathbf{1}} = {\mathbf{G}}{{\mathbf{P}}^{ - 1}}{({{\mathbf{1}}^T}{\mathbf{G}})^T}{\mathbf{1}} = {\mathbf{G1}} = {\mathbf{Q}}
		\end{aligned}
	\end{equation}
	Thus, $\mathbf{Q}$ is a degree matrix of the adjacency matrix ${\mathbf{S}}$.
	
	This shows that Eq.~(\ref{22}) is equivalent to manifold learning in the original space, i.e.,
	\begin{equation}\label{90}
		\begin{aligned}
			\mathop {\min }\limits_{{\mathbf{G}} \in Ind} \sum\limits_{i = 1}^n {\sum\limits_{j = 1}^n {\left\| {{{\mathbf{x}}^i} - {{\mathbf{x}}^j}} \right\|_2^2s_{ij}}} \hfill 
		\end{aligned}
	\end{equation}
	
	With this, Theorem \ref{trm1} is proved.\quad\quad\quad\quad\quad\quad\quad\quad\quad$\square$
\end{proof}

Clearly, from Theorem \ref{trm1}, we can see the equivalence between K-means and manifold learning. Therefore, we use manifold learning to rewrite K-means as Eq.~(\ref{90}), thereby avoiding the estimation of the centroid matrix. Since the manifold structure $\bf{S}$ comprises the label $\bf{G}$, the clustering labels on the same manifold surface are consistent.

\subsection{Objective}
Theorem \ref{trm1} demonstrates that K-means is equivalent to manifold clustering in the original space. Subsequently, we integrate graph embedding and K-means into a unified framework to achieve manifold clustering in the reduced space. Specifically, we use the labels obtained from K-means clustering to construct the manifold structure of the data, which in turn guides the dimensionality reduction process. This approach yields a self-supervised dimensionality reduction clustering framework.

The specific model is as follows:
\begin{equation}\label{model}
	\begin{aligned}
		&	\mathop {\min }\limits_{{\mathbf{W}},{\mathbf{G}}} \sum\limits_{i = 1}^n {\sum\limits_{j = 1}^n {\left\| {{{\mathbf{x}}^i}{\mathbf{W}} - {{\mathbf{x}}^j}}{\mathbf{W}} \right\|_2^2s_{ij}}}  \hfill \\
		& \text{s.t.} \,{\mathbf{G}} \in \text{Ind},\, {{\mathbf{W}}^T}{\mathbf{X}^T\mathbf{B}}{{\mathbf{X}}}{\mathbf{W}} = {\mathbf{I}} \hfill \\ 
	\end{aligned}
\end{equation}
where the manifold structure \textbf{S} is determined by the clustering label $\mathbf{G}$.

The goal of LPP is to perform dimensionality reduction guided by manifold learning. By applying LPP to our framework (\ref{model}), we derive Our-LPP, as shown in Eq.~(\ref{model11}). Specifically, Our-LPP integrates LPP and the K-means clustering algorithm into a unified framework and introduces a self-supervised learning mechanism to guide the dimensionality reduction process. This self-supervised mechanism assesses the similarity between samples based on their label information and neighborhood relationships. If two samples share the same label and are K nearest neighbors, they are considered similar; otherwise, their similarity is considered zero, as illustrated in Figure \ref{KNN}.

The specific model is as follows:
\begin{equation}\label{model11}
	\begin{aligned}
		&	\mathop {\min }\limits_{{\mathbf{W}},{\mathbf{G}}} \sum\limits_{i = 1}^n {\sum\limits_{j = 1}^n {\left\| {{{\mathbf{x}}^i}{\mathbf{W}} - {{\mathbf{x}}^j}}{\mathbf{W}} \right\|_2^2s_{ij}}}  \hfill \\
		& \text{s.t.} \,{\mathbf{G}} \in \text{Ind},\, {{\mathbf{W}}^T}{\mathbf{X}^T\mathbf{D}_S}{{\mathbf{X}}}{\mathbf{W}} = {\mathbf{I}} \hfill \\ 
	\end{aligned}
\end{equation}
where the definition of \textbf{S} is:
\begin{equation}\label{S11}
	\begin{aligned}
		{s_{ij}}=
		\begin{cases}
			\left\langle {{{\mathbf{z}}_i},{{\mathbf{z}}_j}} \right\rangle, & \text{if } {{\mathbf{x}}^i} \in \mathcal{N}_k({{\mathbf{x}}^j}) \text{ or } {{\mathbf{x}}^j} \in \mathcal{N}_k({{\mathbf{x}}^i});\\
			0, & \text{otherwise}.
		\end{cases}
	\end{aligned}
\end{equation}
where ${\mathbf{Z}} = {\mathbf{G}}{{\mathbf{P}}^{ - 1/2}}$, \textbf{P} is a diagonal matrix whose diagonal elements are defined as ${p_{ii}} = \sum\nolimits_{j = 1}^n {{g_{ji}}} $.

\begin{figure}[!t]
	\centering
	\includegraphics[width=0.5\linewidth]{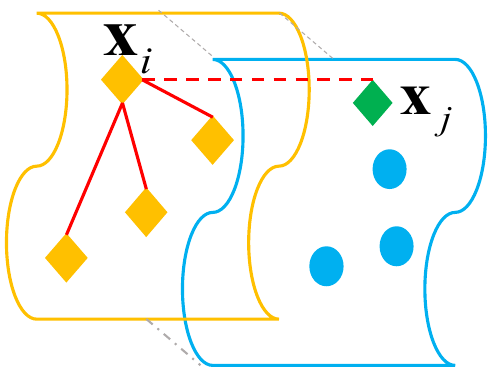}
	\caption{Schematic diagram of label and manifold structure consistency.}
	\label{KNN}
\end{figure}

In addition, our framework (\ref{model}) can be extended to its MFA form, i.e., Our-MFA. The model is formulated as:

\begin{equation}
	\begin{aligned}
		&\mathop {\min }\limits_{\mathbf{W,G}} \frac{\sum\limits_{i = 1}^n \sum\limits_{j = 1}^n \left\| \mathbf{x}^i \mathbf{W} - \mathbf{x}^j \mathbf{W} \right\|_2^2 s_{w(ij)}}{\sum\limits_{i = 1}^n \sum\limits_{j = 1}^n \left\| \mathbf{x}^i \mathbf{W} - \mathbf{x}^j \mathbf{W} \right\|_2^2 s_{b(ij)}} \\
		& \text{s.t. } \mathbf{G} \in Ind, \quad \mathbf{W}^T \mathbf{W} = \mathbf{I}
	\end{aligned}
\end{equation}
where $\mathbf{S}_w$ and $\mathbf{S}_b$ are defined as:

	$\begin{aligned}
		s_{w(ij)} = \left\{ \begin{array}{ll}
			\left\langle \mathbf{z}_i, \mathbf{z}_j \right\rangle, & \text{if } \mathbf{x}^i \in \mathcal{N}_k(\mathbf{x}^j) \text{ or } \mathbf{x}^j \in \mathcal{N}_k(\mathbf{x}^i); \\
			0, & \text{otherwise}.
		\end{array} \right.
	\end{aligned}$\\
	
	$\begin{aligned}
		s_{b(ij)} = \left\{ \begin{array}{ll}
			1 - \left\langle \mathbf{z}_i, \mathbf{z}_j \right\rangle, & \text{if } \mathbf{x}^i \in \mathcal{N}_k(\mathbf{x}^j) \text{ or } \mathbf{x}^j \in \mathcal{N}_k(\mathbf{x}^i); \\
			0, & \text{otherwise}.
		\end{array} \right.
	\end{aligned}$

\subsection{Optimization}
Taking Our-LPP as an example, the optimization of Eq.~(\ref{model11}) focuses on the solution of \textbf{W} and \textbf{G}.

$\bullet$ \textbf{Update \textbf{W} with fixed \textbf{G}:} In this case, the sub-problem for \textbf{W} can be formulated as follows:
\begin{equation}\label{W}
	\begin{aligned}
		& \mathop{\min}\limits_{\mathbf{W}} \sum\limits_{i=1}^n \sum\limits_{j=1}^n \left\| \mathbf{x}^i\mathbf{W} - \mathbf{x}^j\mathbf{W} \right\|_2^2 s_{ij} \\
		& \text{s.t.} \quad \mathbf{W}^T \mathbf{X}^T \mathbf{D}_S \mathbf{X} \mathbf{W} = \mathbf{I}
	\end{aligned}
\end{equation}

As derived in \cite{HeN03}, Eq.~(\ref{W}) can be rewritten as:
\begin{equation}\label{eqW1}
	\begin{aligned}
		& \mathop{\min}\limits_{\mathbf{W}} \, \text{tr}(\mathbf{W}^T \mathbf{X}^T \mathbf{L}_S \mathbf{X} \mathbf{W}) \\
		& \text{s.t.} \quad \mathbf{W}^T \mathbf{X}^T \mathbf{D}_S \mathbf{X} \mathbf{W} = \mathbf{I}
	\end{aligned}
\end{equation}
where $\mathbf{D}_S$ is a diagonal matrix with ${d_S}_{ii} = \sum\limits_{j=1}^n s_{ij}$, and $\mathbf{L}_S = \mathbf{D}_S - \mathbf{S}$.

Further, Eq.~(\ref{eqW1}) can be transformed to solve:
\begin{equation}\label{xx}
	\begin{aligned}
		\mathop{\min}\limits_{\mathbf{W}^T \mathbf{W} = \mathbf{I}} \, \text{tr}(\mathbf{W}^T \mathbf{X}^T (\mathbf{L}_S - \eta \mathbf{D}_S) \mathbf{X} \mathbf{W})
	\end{aligned}
\end{equation}
where $\eta$ is an adjustable parameter. The optimal value of \textbf{W} is then composed of the eigenvectors corresponding to the first $m$ smallest eigenvalues of $\mathbf{X}^T (\mathbf{L}_S - \eta \mathbf{D}_S) \mathbf{X}$. Algorithm \ref{A1} presents the pseudo-code for solving the problem (\ref{W}).

\begin{algorithm}[ht]
	\caption{Pseudo-Code for Solving Problem (\ref{W})}
	\begin{algorithmic}[1]\label{A1}
		\REQUIRE Data matrix ${\mathbf{X}}\in \mathbb{R}^{n\times d}$; Cluster assignment \\ matrix $\mathbf{G} \in \mathbb{R}^{n\times c}$; Number of neighbors $k$.\\
		\ENSURE Projection matrix $\mathbf{W}$.\\
		\WHILE{\emph{not converge}}
		\STATE Determine $\mathbf{S}$ ;
		\STATE Take the eigenvectors corresponding to the first $m$ smallest eigenvalues of ${\mathbf{X}}^T({{\mathbf{L}}_S} - \eta {{\mathbf{D}}_S}){\mathbf{X}}$ to form \\
		\textbf{W};\\
		\ENDWHILE
		\STATE \textbf{return}: The projection matrix $\mathbf{W}$.
	\end{algorithmic}
\end{algorithm}

$\bullet$ \textbf{Update \textbf{G} with fixed \textbf{W}: } The sub-problem for \textbf{G} can be formulated as follows:
\begin{equation}\label{derive_F}
	\begin{aligned}
		\mathop {\min }\limits_{{\mathbf{G}} \in Ind} \sum\limits_{i = 1}^n \sum\limits_{j = 1}^n \left\| {{{\mathbf{x}}^i\textbf{W}} - {{\mathbf{x}}^j\textbf{W}}} \right\|_2^2 s_{ij}
	\end{aligned}
\end{equation}
When solving \textbf{G}, it can be further expressed as
\begin{equation}\label{model1}
	\begin{aligned}
		&\mathop{\min }\limits_{{\mathbf{G}} \in Ind}\sum\limits_i \sum\limits_j {d_{ij}} \langle \textbf{g}_i, \textbf{g}_j \rangle = \mathop{\min }\limits_{{\mathbf{G}} \in Ind} \text{tr}(\textbf{G}^T\textbf{D}\textbf{G}\textbf{P}^{-1})
	\end{aligned}
\end{equation}
where $d_{ij} = \left\| {{{\mathbf{x}}^i}{\mathbf{W}} - {{\mathbf{x}}^j}{\mathbf{W}}} \right\|_2^2$.

The model (\ref{model1}) is difficult to solve, so to optimize it and ensure class balance after clustering, we introduce Theorem \ref{theorem4}.
\begin{theorem}\label{theorem4}
	Given $b_1 + b_2 + \ldots + b_c = N$, where $b_j \geq 0$ represents the number of samples in the j-th cluster, Eq.~(\ref{the4}) attains its maximum value when $b_1 = b_2 = \ldots = b_c$. In this case, $\bf{G}$ is discrete and presents a balanced class distribution.
	\begin{equation}\label{the4}
		\begin{aligned}
			&\mathop {\rm{max}}\limits_{\bf{G} \geq 0, \bf{G}\bf{1} = \bf{1}}  \left\| {\bf{G}^T} \right\|_{2,1}
		\end{aligned}
	\end{equation}
\end{theorem}	

\begin{proof}
	\begin{equation}\label{L21}
		\begin{aligned}
			\left\| {\bf{G}^T} \right\|_{2,1} = \sum\nolimits_j \left( g_j^T g_j \right)^{1/2} =  \sum\nolimits_{j = 1}^c b_j
		\end{aligned}
	\end{equation}
	where $	b_j = \left( \sum\nolimits_i g_{ij}^2 \right)^{1/2} $
	
	Let $\mathbf{b} = [b_1, b_2, \ldots ,b_c]^T \in \mathbb{R}^{c\times 1}$ and $\mathbf{q} = [1, 1, \ldots ,1]^T \in \mathbb{R}^{c\times 1}$. According to the Cauchy-Schwarz inequality, 
	\begin{equation}\label{L3}
		\begin{aligned}
			\left| \left\langle \bf{b}, \bf{q} \right\rangle \right| &\le \left\| \bf{b} \right\|_2 \left\| \bf{q} \right\|_2 \\
			\Rightarrow \sum\limits_j b_j &\le \left\| \bf{b} \right\|_2 \left( 1 + 1 + \ldots + 1 \right)^{1/2} \\
			\Rightarrow \sum\limits_j b_j &\le \left\| \bf{b} \right\|_2 {c}^{1/2}
		\end{aligned}
	\end{equation}
	The equality holds if and only if $b_1 = b_2 = \ldots = b_c$. 
	
Considering,
	\begin{equation}\label{L6}
		\begin{aligned}
			\max \left\| \bf{b} \right\|_2 \Rightarrow \max \left\| \bf{b} \right\|_2^2 = \max \sum\limits_i \sum\limits_j g_{ij}^2
		\end{aligned}
	\end{equation}
	
	Obviously, all rows of $\bf{G}$ are independent, thus
	\begin{equation}\label{L7}
		\begin{aligned}
			\mathop {\rm{max}}\limits_{ 0 \le g_{ij} \le 1, \sum\limits_j g_{ij} = 1}  \sum \limits_{j = 1}^c g_{ij}^2,
			\forall i
		\end{aligned}
	\end{equation}
	
	The solution to the maximization problem (\ref{L7}) is achieved when $\bf{g}_i$ has only one element equal to 1 and the rest are 0, and the maximum value is 1. Thus, the problem (\ref{L7}) only reaches a maximum when $\bf{G}$ is a discrete label matrix.
	
	In this case, ${\bf{P} = \bf{G}^T \bf{G}} \in \mathbb{R}^{c\times c}$ is a diagonal matrix whose i-th diagonal element is the number of samples in the i-th cluster, hence:
	\begin{equation}\label{L8}
		\begin{aligned}
			b_j = \sqrt{g_j^T g_j} = \sqrt{n_i}
		\end{aligned}
	\end{equation}
	where $n_i$ is the number of samples in the j-th cluster.
	
	According to $b_1 = b_2 = \ldots = b_c$, we have $\sqrt{n_1} = \sqrt{n_2} = \ldots = \sqrt{n_c}$.
	\quad\quad\quad\quad\quad\quad\quad\quad\quad\quad\quad\quad\quad\quad\quad$\square$
\end{proof}

According to Theorem \ref{theorem4}, we convert Eq.~(\ref{model1}) to

\begin{equation}\label{F}
	\begin{aligned}
		\mathop {\rm{min}}\limits_{\textbf{G} \geq 0, \textbf{G}\textbf{1} = \textbf{1}} \text{tr}(\textbf{G}^T \textbf{D} \textbf{G}) - \beta \left\| \textbf{G}^T \right\|_{2,1}
	\end{aligned}
\end{equation}

To simplify the optimization process, we set $f(\textbf{G})=\| \textbf{G}^T \|_{2,1}$ and perform a first-order Taylor expansion of $f(\textbf{G})$ at $\textbf{G}^{(t)}$, yielding:
\begin{equation}\label{Taylor}
	f(\textbf{G}) = f(\textbf{G}^{(t)}) + \langle \nabla f(\textbf{G}^{(t)}), \textbf{G} - \textbf{G}^{(t)} \rangle,
\end{equation}
where $\textbf{G}^{(t)}$ is the solution at iteration $t$, and $\nabla f(\textbf{G}^{(t)})$ is the derivative of $\| \textbf{G}^T \|_{2,1}$.

Denote the derivative of $\| \textbf{G}^T \|_{2,1}$ by $\textbf{H}$, we have
\begin{equation}\label{H}
	\textbf{H} = \frac{\partial {\left\| \textbf{G}^T \right\|_{2,1}}}{\partial \textbf{G}} =  \textbf{G} \cdot \text{diag}\left(\frac{1}{{\left\| \textbf{g}_1 \right\|}}, \ldots, \frac{1}{{\left\| \textbf{g}_c \right\|}}\right).
\end{equation}

Neglecting constants in Eq.~(\ref{Taylor}), we can solve Eq.~(\ref{F}) iteratively by:
\begin{equation}\label{Taylor_solve}
	\begin{aligned}
		\textbf{G}^{(t+1)} 
		&= \mathop{\text{argmin}}\limits_{\textbf{G}} \text{tr}(\textbf{G}^T \textbf{D} \textbf{G}) - \beta \text{tr}(\textbf{H}^T \textbf{G})
	\end{aligned}
\end{equation}

Thus, we approximate Eq.~(\ref{F}) by Eq.~(\ref{au}), and $\textbf{G}$ is updated by solving the following problem:
\begin{equation}\label{au}
	\begin{aligned}
		\min_{\textbf{G} \cdot \textbf{1}=\textbf{1}, \textbf{G} \geqslant \mathbf{0}} \text{tr}(\textbf{G}^T \textbf{D} \textbf{G}) - \beta \text{tr}(\textbf{H}^T \textbf{G}).
	\end{aligned}
\end{equation}

Let $\mathbf{G} = \begin{bmatrix}
	\mathbf{g}^{i} \\
	\mathbf{G}_0 \\
\end{bmatrix}$, $\mathbf{D} = \begin{bmatrix}
	d_{ii} &  \mathbf{d}_{i0}^{T}  \\
	\mathbf{d}_{i0} & \mathbf{D}_0  \\
\end{bmatrix}$, where $\mathbf{G}_0 \in \mathbb{R}^{(N-1) \times K}$, $\mathbf{d}_{i0} \in \mathbb{R}^{(N-1) \times 1}$, and $\mathbf{D}_0 \in \mathbb{R}^{(N-1) \times (N-1)}$. Similarly, $\mathbf{H} = \begin{bmatrix}
\mathbf{h}^{i} \\
\mathbf{H}_0 \\
\end{bmatrix}$. We have:
\begin{equation}
	\begin{aligned}
		&\textbf{G}^T \textbf{D} \textbf{G} -  \beta \textbf{H}^T \textbf{G} \\
		&= \left(\begin{bmatrix}
			(\mathbf{g}^{i})^T & (\mathbf{G}_0)^T \\
		\end{bmatrix}
		\begin{bmatrix}
			d_{ii} &  \mathbf{d}_{i0}^T  \\
			\mathbf{d}_{i0} & \mathbf{D}_0  \\
		\end{bmatrix}  
		-  \beta \begin{bmatrix}
			(\mathbf{h}^{i})^T & (\mathbf{H}_0)^T \\
		\end{bmatrix}\right)
		\begin{bmatrix}
			\mathbf{g}^{i} \\
			\mathbf{G}_0 \\
		\end{bmatrix} \\
		&= (\mathbf{g}^{i})^T d_{ii} \mathbf{g}^{i} + (\mathbf{G}_0)^T \mathbf{d}_{i0} \mathbf{g}^{i} + 
		(\mathbf{g}^{i})^T \mathbf{d}_{i0}^T \mathbf{G}_0 \\
		&\quad\quad + (\mathbf{G}_0)^T \mathbf{D}_0 \mathbf{G}_0 - \beta \left((\mathbf{h}^{i})^T \mathbf{g}^{i} + (\mathbf{H}_0)^T \mathbf{G}_0 \right).
	\end{aligned}
\end{equation}

Removing terms unrelated to $\mathbf{g}^i$, and using the properties of the trace operation, we have:
\begin{equation}
	\begin{aligned}
		&\text{tr}\left(\textbf{G}^T \textbf{D} \textbf{G} -  \beta \textbf{H}^T \textbf{G}\right) \\
		&\quad\quad= \text{tr}\left((\mathbf{g}^{i})^T d_{ii} \mathbf{g}^{i} +  2 \mathbf{g}^{i} \mathbf{G}_0^T \mathbf{d}_{i0} - \beta \mathbf{g}^{i} (\mathbf{h}^{i})^T \right) \\
		&\quad\quad= \mathbf{g}^{i} (\mathbf{g}^{i})^T d_{ii} +  \mathbf{g}^{i} \mathbf{f},
	\end{aligned}
\end{equation}
where $\mathbf{f} = 2 \mathbf{G}_0^T \mathbf{d}_{i0} - \beta (\mathbf{h}^{i})^T$.

Thus, the problem of updating the $i$-th row of $\textbf{G}$ becomes:
\begin{equation}\label{SolveYFinal2}
	\begin{aligned}
		& \mathop{\textrm{min}}\limits_{\mathbf{g}^{i} \cdot \textbf{1} = \textbf{1}} \mathbf{g}^{i} (\mathbf{g}^{i})^T d_{ii} +  \mathbf{g}^{i} \mathbf{f}.
	\end{aligned}
\end{equation}

As $d_{ii} = 0 \; (i = 1, 2, \dots, N)$, Eq.~(\ref{SolveYFinal2}) reduces to:
\begin{equation}\label{SolveYFinal}
	\begin{aligned}
		\mathop{\textrm{min}}\limits_{\mathbf{g}^i }  \mathbf{g}^{i} \left(2 \mathbf{G}^T \mathbf{d}_{i} - \beta (\mathbf{h}^{i})^T \right)
	\end{aligned}
\end{equation}
where $\mathbf{d}_i$ is the $i$-th column of $\textbf{D}$, with $d_{ii} = 0$. $\textbf{G}$ denotes the solution before $\mathbf{g}^i$ is updated. Then, the solution of $\mathbf{g}^i$ can be written as:
\begin{equation}\label{solveF}
	\begin{aligned}
		g_{ib} = \begin{cases}
			1, & b = \mathop{\textrm{arg min}}\limits_{j} \left(2 \mathbf{G}^T \mathbf{d}_{i} - \beta (\mathbf{h}^{i})^T \right)_j \\
			0, & \textrm{otherwise}.
		\end{cases}
	\end{aligned}
\end{equation}

Algorithm \ref{A2} concludes the pseudo-code for solving the problem (\ref{F}).
\begin{algorithm}[ht]
	\caption{Pseudo-Code for Solving Problem (\ref{F})}
	\begin{algorithmic}[1]\label{A2}
		\REQUIRE Data matrix ${\mathbf{X}}\in \mathbb{R}^{n\times d}$; Projection matrix ${\mathbf{W}}\in \mathbb{R}^{d\times m}$; Number of clusters $c$.\\
		\ENSURE Cluster assignment matrix $\mathbf{G}$.\\
		\WHILE{\emph{not converge}}
		\STATE  Update  $\mathbf{H}$ by Eq.~\eqref{H}; 
		\STATE Update $\mathbf{G}$ by (\ref{solveF}) row by row;\\
		\ENDWHILE
		\STATE \textbf{return}: The cluster assignment matrix $\mathbf{G}$.
	\end{algorithmic}
\end{algorithm}

Finally, we conclude the pseudo-code on Algorithm \ref{A3}.
\begin{algorithm}[ht]
	\caption{Pseudo-Code for our algorithm}
	\begin{algorithmic}[1]\label{A3}
		\REQUIRE Data matrix ${\mathbf{X}}\in \mathbb{R}^{n\times d}$; Number of neighbors \\ $k$; Number of clusters $c$.\\
		\ENSURE Cluster assignment matrix $\mathbf{G}$.\\
		\STATE \textbf{Initialize}: $\mathbf{G}$.\\
		\WHILE{\emph{not converge}}
		\STATE Update $\mathbf{W}$ by Algorithm \ref{A1};\\
		\STATE Update $\mathbf{G}$ by Algorithm \ref{A2};\\
		\ENDWHILE
		\STATE \textbf{return}: The cluster assignment matrix $\mathbf{G}$.
	\end{algorithmic}
\end{algorithm}

\textbf{Time Complexity Analysis:} 
The time complexity of solving \textbf{W} using eigen-decomposition is $\bm{\mathcal{O}}(d^3)$. 
For \textbf{G}, the complexity is primarily due to the computation of the distance matrix \textbf{D} and the optimization of \textbf{H} and \textbf{G} row by row. The computation of \textbf{D} has a complexity of $\bm{\mathcal{O}}(n^2dm)$. Optimizing \textbf{H} requires $\bm{\mathcal{O}}(nc^2)$ according to Eq.~\eqref{H}, and solving \textbf{G} row by row has a complexity of $\bm{\mathcal{O}}(n^2)$. 
Thus, the overall time complexity of Algorithm \ref{A3} is $\bm{\mathcal{O}}(t_3(t_1d^3 + t_2n^2dm))$, where $t_1$ represents the number of iterations of Algorithm \ref{A1}, $t_2$ denotes the number of iterations of Algorithm \ref{A2}, and $t_3$ corresponds to the number of iterations of Algorithm \ref{A3}.

\begin{table*}[ht]
	\centering
	\setlength{\tabcolsep}{3pt}
	\caption{ACC results on seven benchmark datasets, with the highest value represented in bold.}\label{res_acc}
	\begin{center}
		\resizebox{2.1\columnwidth}{!}{
			\begin{tabular}{l c c c c c c c c c c| c c}
				\toprule
				Datasets &K-means	&Ksum	&Ksum-x	&RKM	&CDKM	&\makecell[c]{PCA-KM} &\makecell[c]{LPP-KM} &\makecell[c]{LDA-KM} &\makecell[c]{Un-RTLDA} &\makecell[c]{Un-TRLDA} &\makecell[c]{Our-LPP}  &\makecell[c]{Our-MFA}\\
				\midrule
				AR	&0.251	&0.297	&0.245	&0.264	&0.265	&0.282	&0.449	&0.258	&0.561	&0.272	&0.560				
				&\textbf{0.601}\\
				JAFFE &0.709	&0.879	&0.893	&0.831	&0.711	&0.887	&0.977	&0.958	&0.967	&0.916		&\textbf{1.000}	&\textbf{1.000}\\
MSRC &0.605	&0.752	&0.685	&0.629	&0.666	&0.691	&0.724	&0.729	&0.671	&0.752	&0.886		&\textbf{0.929}\\	
ORL&0.520	&0.634	&0.588	&0.500	&0.551	&0.553	&0.550	&0.568	&0.663	&0.608	&\textbf{0.895}	&0.848\\
UMIST&0.434	&0.421	&0.430	&0.421	&0.421	&0.464	&0.457	&0.499	&0.497	&0.501	&0.817	&\textbf{0.837}\\	
USPS&0.649	&0.766	&0.724	&0.667	&0.642	&0.778	&0.686	&0.759	&0.744	&0.791	&0.826	&0.782\\
Yaleface&0.381	&0.434	&0.442	&0.449	&0.396	&0.473	&0.400	&0.454	&0.473	&0.485	&\textbf{0.655}		&0.588\\
				\bottomrule
		\end{tabular}}
	\end{center}
\end{table*}
\begin{table*}[ht]
	\centering
	\setlength{\tabcolsep}{3pt}
	\caption{NMI results on seven benchmark datasets, with the highest value represented in bold.}\label{res_NMI}
	\begin{center}
		\resizebox{2.1\columnwidth}{!}{
			\begin{tabular}{l c c c c c c c c c c| c c}
				\toprule
				Datasets &K-means	&Ksum	&Ksum-x	&RKM	&CDKM	&\makecell[c]{PCA-KM} &\makecell[c]{LPP-KM} &\makecell[c]{LDA-KM} &\makecell[c]{Un-RTLDA} &\makecell[c]{Un-TRLDA} &\makecell[c]{Our-LPP}  &\makecell[c]{Our-MFA}\\
				\midrule
				AR	&0.557	&0.596	&0.568	&0.575	&0.570	&0.583	&0.704	&0.555	&0.801	&0.567	&0.784	
				&\textbf{0.833}\\
				JAFFE &0.801	&0.876	&0.901	&0.816	&0.798	&0.914	&0.974	&0.948	&0.963	&0.923	&\textbf{1.000}		&\textbf{1.000}\\
				MSRC	&0.528	&0.611	&0.575	&0.561	&0.569	&0.603	&0.625	&0.584	&0.548	&0.638	&0.774		&\textbf{0.855}\\
				ORL&0.723	&0.794	&0.769	&0.714	&0.753	&0.733	&0.692	&0.746	&0.816	&0.772	&\textbf{0.929}		&0.909\\
				UMIST	&0.641	&0.619	&0.638	&0.596	&0.640	&0.625	&0.662	&0.674	&0.693	&0.685	&0.890		&\textbf{0.922}\\
				USPS&0.615	&0.662	&0.608	&0.587	&0.606	&0.667	&0.653	&0.656	&0.642	&0.676	&0.701	&0.674\\
				Yaleface	&0.439	&0.498	&0.502	&0.510	&0.478	&0.508	&0.430	&0.508	&0.521	&0.527	&\textbf{0.664}		&0.627\\
				\bottomrule
		\end{tabular}}
	\end{center}
\end{table*}
\begin{table*}[!h]
	\centering
	\setlength{\tabcolsep}{3pt}
	\caption{Purity results on seven benchmark datasets, with the highest value represented in bold.}\label{res_Purity}
	\begin{center}
			\resizebox{2.1\columnwidth}{!}{
					\begin{tabular}{l c c c c c c c c c c| c c}
							\toprule
							Datasets &K-means	&Ksum	&Ksum-x	&RKM	&CDKM	&\makecell[c]{PCA-KM} &\makecell[c]{LPP-KM} &\makecell[c]{LDA-KM} &\makecell[c]{Un-RTLDA} &\makecell[c]{Un-TRLDA} &\makecell[c]{Our-LPP}  &\makecell[c]{Our-MFA}\\
							\midrule
							AR	&0.275	&0.369	&0.324	&0.322	&0.286	&0.299	&0.493	&0.278	&0.597	&0.294	&0.591
							&\textbf{0.628}\\
							JAFFE &0.746	&0.879	&0.898	&0.831	&0.744  &0.887	&0.977	&0.958	&0.967	&0.916	&\textbf{1.000}		&\textbf{1.000}\\
							MSRC&0.628	&0.752	&0.691	&0.633	&0.680	&0.719	&0.724	&0.729	&0.691	&0.752	&0.886		&\textbf{0.922}\\
							ORL&0.571	&0.656	&0.606	&0.520	&0.609	&0.603	&0.588	&0.610	&0.698	&0.645	&\textbf{0.896}		&0.858\\
							UMIST&0.511	&0.455	&0.472	&0.440	&0.504  &0.518	&0.564	&0.553	&0.569	&0.562	&0.854		&\textbf{0.884}\\
							USPS&0.680	&0.768	&0.724	&0.683	&0.675	&0.778	&0.722	&0.759	&0.744	&0.791	&0.826		&0.782\\
							Yaleface&0.403	&0.473	&0.492	&0.485	&0.419	&0.473	&0.424	&0.473	&0.497	&0.503	&\textbf{0.661}	&0.600\\
							\bottomrule
					\end{tabular}}
		\end{center}
\end{table*}

\section{Experiments}\label{Experiments}
\subsection{Experimental Settings}
Experiments are conducted on a Windows 10 desktop computer equipped with a 2.40 GHz Intel Xeon Gold 6240R CPU, 64 GB RAM, and MATLAB R2020b (64-bit).
\subsubsection{Datasets}
We validate the clustering performance of our algorithm using seven datasets, including:
(1) AR~\cite{AR} contains 120 face classes with 3,120 images.
(2) JAFFE~\cite{JAFFE} consists of 213 images displaying different facial expressions from 10 Japanese female subjects.
(3) MSRC\_V2~\cite{WinnJ05} includes 7 object classes with 210 images. We selected the 576-D HOG feature as the single view dataset.
(4) ORL~\cite{ORL} contains 400 facial images from 40 individuals.
(5) UMIST~\cite{UMIST} consists of 564 facial images from 20 individuals.
(6) USPS~\cite{USPS} consists of 3,000 handwritten digit images for each of the 10 digits from 0 to 9.
(7) Yaleface~\cite{Yaleface} contains 165 grayscale images in GIF format from 15 individuals.

\subsubsection{Comparison Algorithms}
We selected several popular K-means algorithms, including K-means, Ksum~\cite{PeiC00023}, Ksum-x~\cite{PeiC00023}, RKM~\cite{LinHX19}, and CDKM~\cite{NieXWWLL22}. Additionally, we included some dimensionality reduction clustering algorithms, both two-step algorithms: PCA-KM~\cite{6793549}, LPP-KM~\cite{HeN03}, LDA-KM~\cite{DingL07}, and one-step algorithms: Un-RTLDA~\cite{WangWNLYW21}, Un-TRLDA~\cite{WangWNLYW21}.
\subsubsection{Evaluation Metrics}
We use three commonly used clustering evaluation metrics to assess the performance of our algorithms: Accuracy (ACC), Normalized Mutual Information (NMI), and Purity.
\subsection{Clustering Performance}

As shown in Table \ref{res_acc}, Table \ref{res_NMI}, and Table \ref{res_Purity}, the clustering algorithms with dimensionality reduction outperform those in the original space in terms of clustering performance. This suggests that dimensionality reduction plays a crucial role in clustering, likely by reducing noise and emphasizing key features.
Among dimensionality reduction techniques, the clustering performance of one-step dimensionality reduction algorithms surpasses that of two-step dimensionality reduction clustering algorithms. The two-step algorithms extract data features through dimensionality reduction and then use them for clustering; however, these features may not be ideal, affecting the clustering outcome. The performance of LPP-KM, based on manifold learning, exceeds that of PCA-KM, which relies on global linear transformation. LPP-KM retains the local manifold structure of the data, which often represents the intrinsic geometric structure of the data space better than global linear methods like PCA.
Moreover, Our-LPP and Our-MFA outperform the one-step clustering algorithm based on LDA. This superior performance is attributed to Our-LPP's ability to preserve the local manifold structure during dimensionality reduction, allowing for a more accurate representation of inherent data clustering. Our-MFA leverages the boundary Fisher criterion to find the optimal projection direction, resulting in more compact homogeneous and dispersed heterogeneous neighborhood samples. Considering boundary information helps maintain data separability after dimensionality reduction, thereby enhancing the clustering effect.

\subsection{Parameter Analysis}

For simplicity, we conducted a detailed experimental analysis on Our-LPP.
We conducted experiments on the parameter $k$ for ORL, MSRC, UMIST, and Yaleface to verify the effectiveness of determining sample similarity based on labels and neighborhood relationships. Specifically, Figure \ref{figk} visually depicts the changes in clustering performance as $k$ varies from $2$ to $n/c$. The results indicate that the clustering performance for ORL, MSRC, and Yaleface typically improves as $k$ increases, reaching a peak slightly below $n/c$ before declining again. For UMIST, the clustering performance peaks at $k=2$ and gradually decreases as $k$ increases. A higher $k$ value captures more neighborhood information, but if $k$ becomes too large, the neighborhood range expands excessively, causing samples from different classes to group together, which leads to a drop in performance.

\begin{figure}[ht]
	\centering
	\begin{subfigure}{0.47 \linewidth}
		\includegraphics[width=1.0\linewidth]{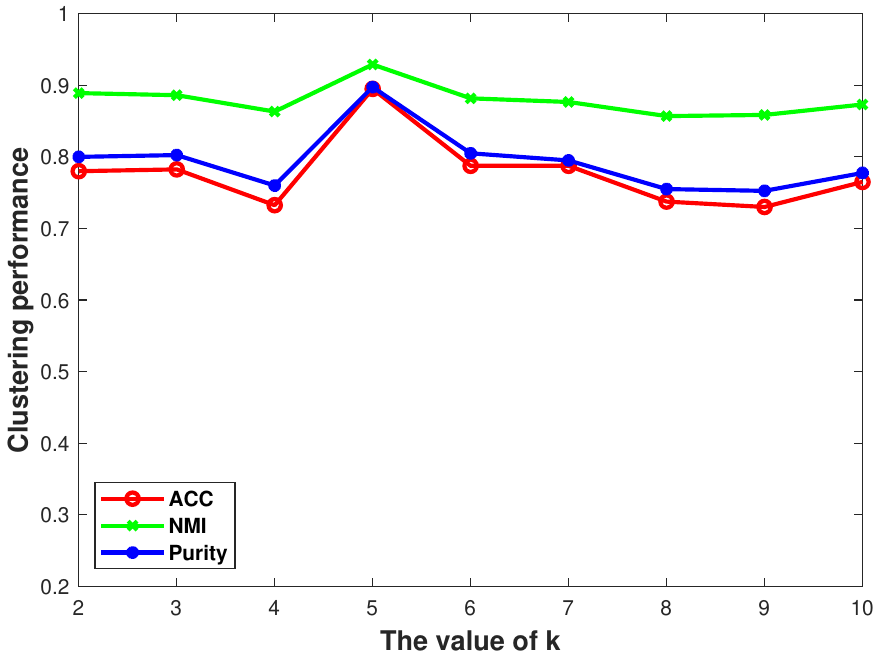}
		\caption{ORL}
	\end{subfigure}
	\begin{subfigure}{0.47 \linewidth}
		\includegraphics[width=1.0\linewidth]{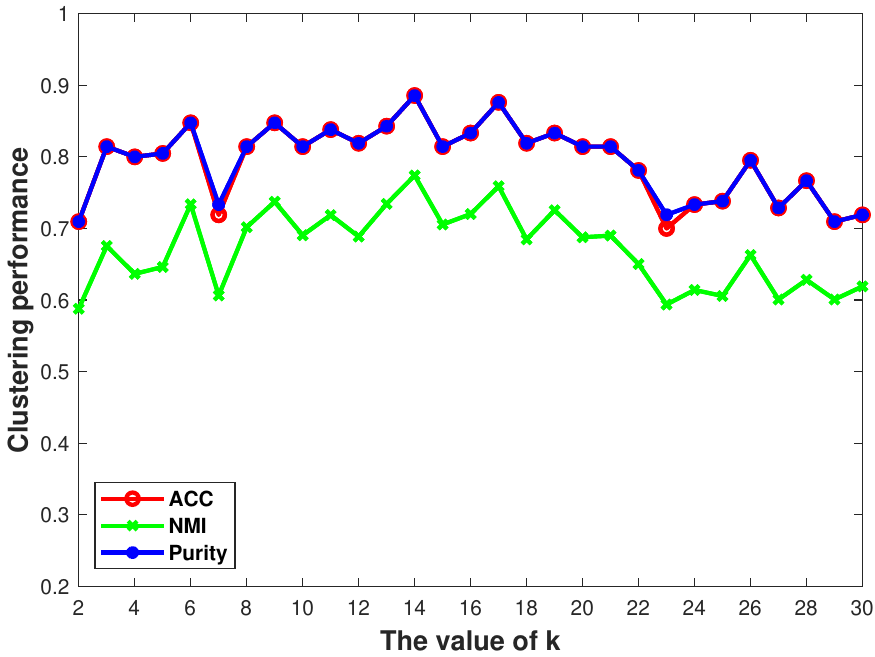}
		\caption{MSRC}
	\end{subfigure}
	\begin{subfigure}{0.47 \linewidth}
		\includegraphics[width=1.0\linewidth]{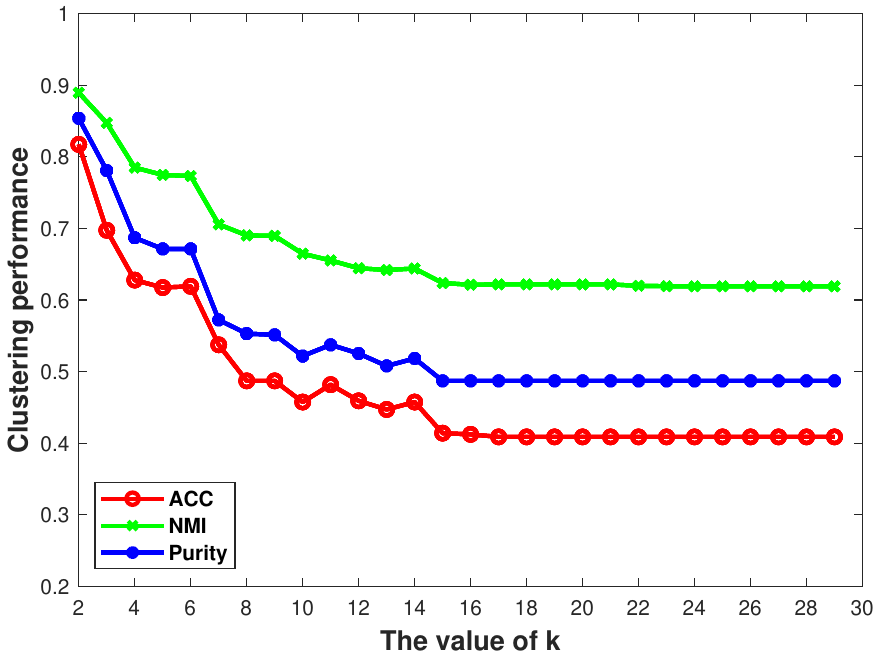}
		\caption{UMIST}
	\end{subfigure}
	\begin{subfigure}{0.47 \linewidth}
		\includegraphics[width=1.0\linewidth]{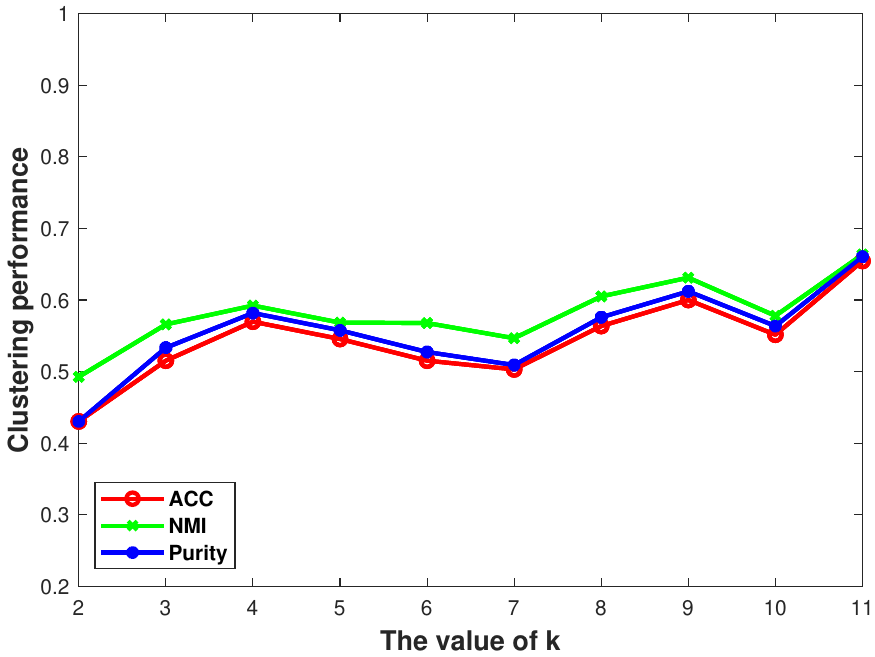}
		\caption{Yaleface}
	\end{subfigure}
	\caption{The effect of $k$ on clustering performance.}
	\label{figk}
\end{figure}
To assess the impact of dimensionality reduction on clustering accuracy, we conducted experiments using the ORL, MSRC, UMIST, and Yaleface datasets. These experiments evaluated how varying the number of dimensions affects clustering outcomes, as illustrated in Figure \ref{figdim}. 

The results indicate that as the data dimensionality increases, clustering performance initially improves gradually. However, beyond a certain point, further increases in dimensionality lead to either stabilization or a decline in clustering effectiveness. This suggests that insufficient dimensionality may result in inadequate feature extraction, thereby constraining the performance of clustering algorithms. Conversely, excessively high dimensionality can introduce redundancy or noise, adversely affecting the clustering process.

\begin{figure}[ht]
	\centering
	\begin{subfigure}{0.47 \linewidth}
		\includegraphics[width=1.0\linewidth]{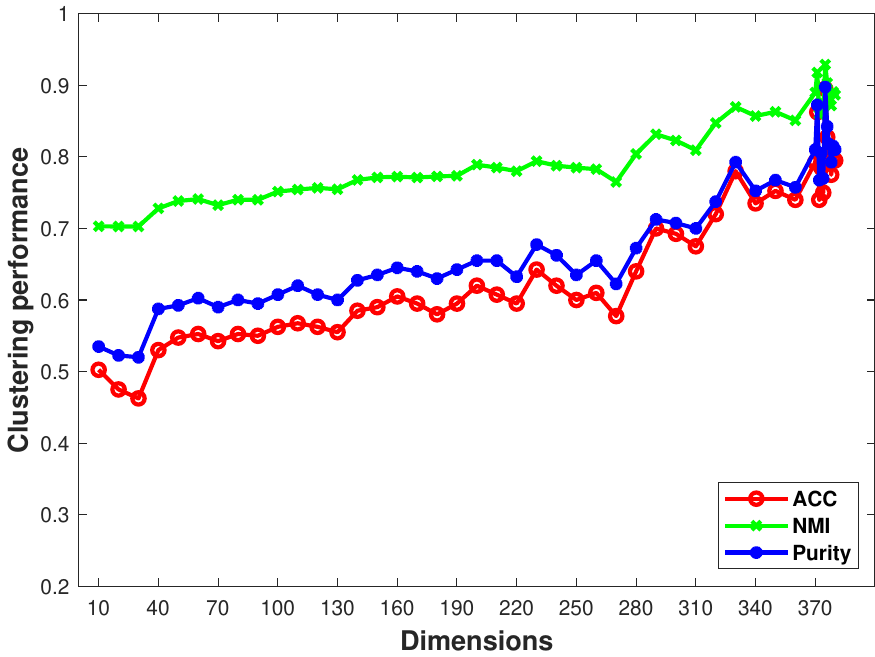}
		\caption{ORL}
	\end{subfigure}
	\begin{subfigure}{0.47 \linewidth}
		\includegraphics[width=1.0\linewidth]{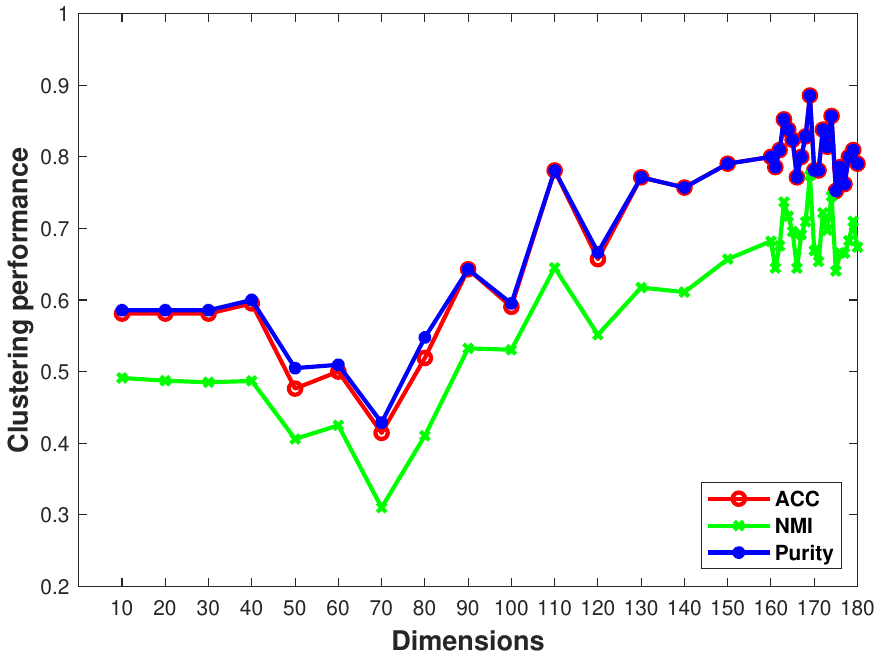}
		\caption{MSRC}
	\end{subfigure}
	\begin{subfigure}{0.47 \linewidth}
		\includegraphics[width=1.0\linewidth]{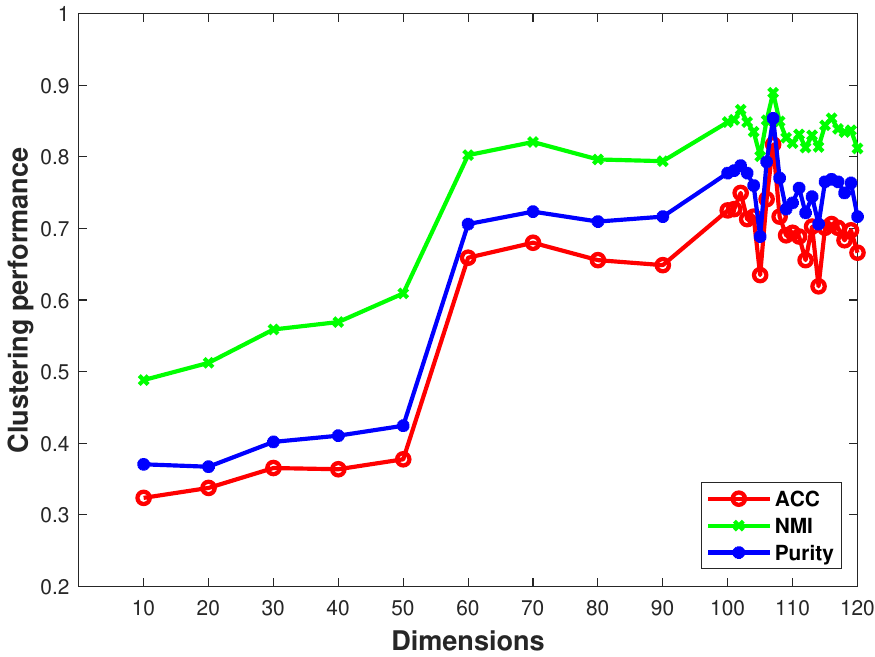}
		\caption{UMIST}
	\end{subfigure}
	\begin{subfigure}{0.47 \linewidth}
		\includegraphics[width=1.0\linewidth]{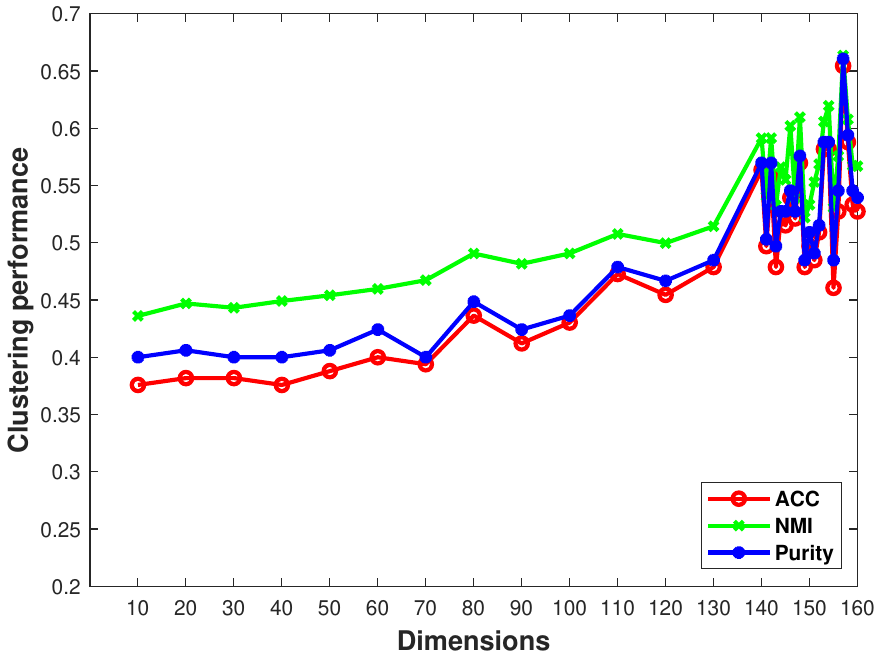}
		\caption{Yaleface}
	\end{subfigure}
	\caption{The effect of dimension on clustering performance.}
	\label{figdim}
\end{figure}

\subsection{Verification of Convergence}
In this section, we verify the convergence of Our-LPP. Figure \ref{figcon} shows the clustering performance and the objective function value as a function of the number of iterations for the ORL, MSRC, UMIST, and Yaleface datasets. As observed, the clustering performance gradually improves, and the objective function value steadily decreases with an increasing number of iterations. After a few iterations, the objective function value and the clustering performance stabilize, indicating convergence. Furthermore, our algorithm converges quickly, demonstrating its strong convergence properties.
\begin{figure}[!ht]
	\centering
	\begin{subfigure}{0.47 \linewidth}
		\includegraphics[width=1.0\linewidth]{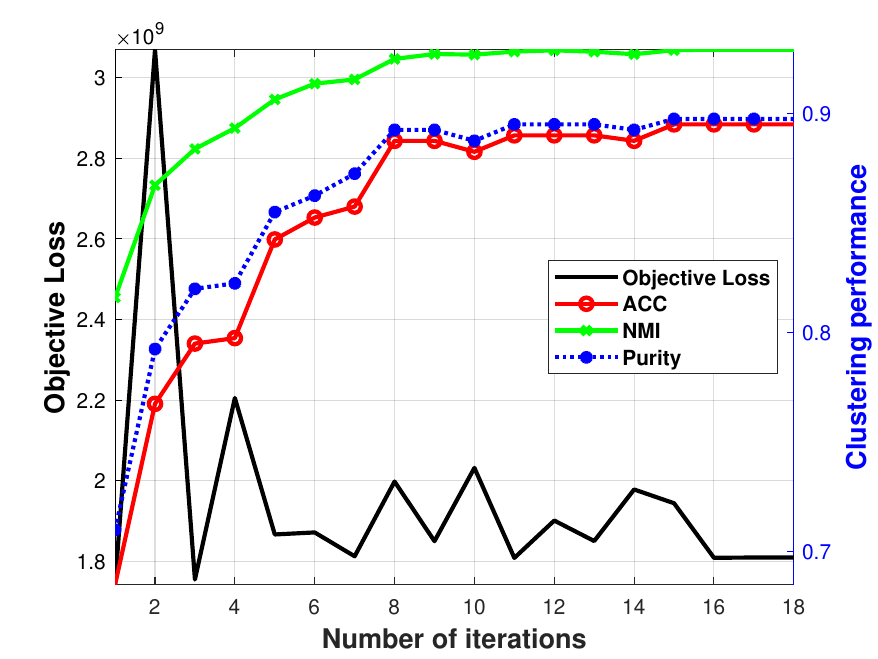}
		\caption{ORL}
	\end{subfigure}
	\begin{subfigure}{0.47 \linewidth}
		\includegraphics[width=1.0\linewidth]{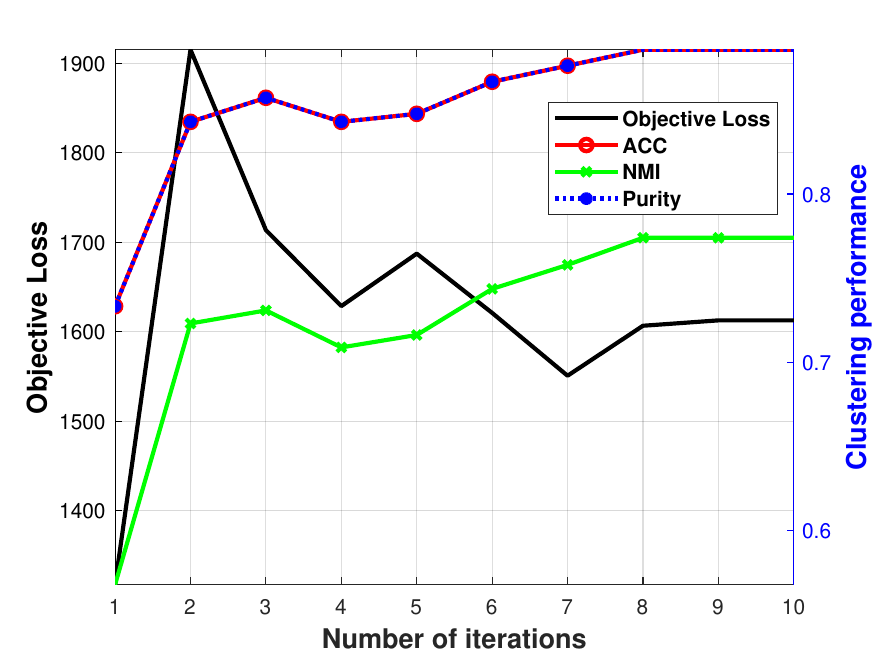}
		\caption{MSRC}
	\end{subfigure}
	\begin{subfigure}{0.47 \linewidth}
		\includegraphics[width=1.0\linewidth]{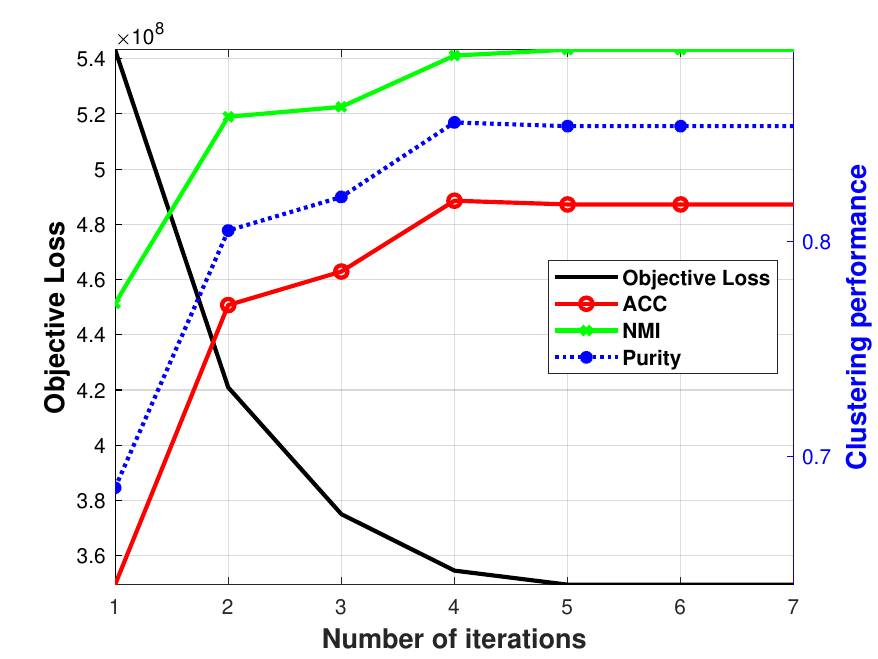}
		\caption{UMIST}
	\end{subfigure}
	\begin{subfigure}{0.47 \linewidth}
		\includegraphics[width=1.0\linewidth]{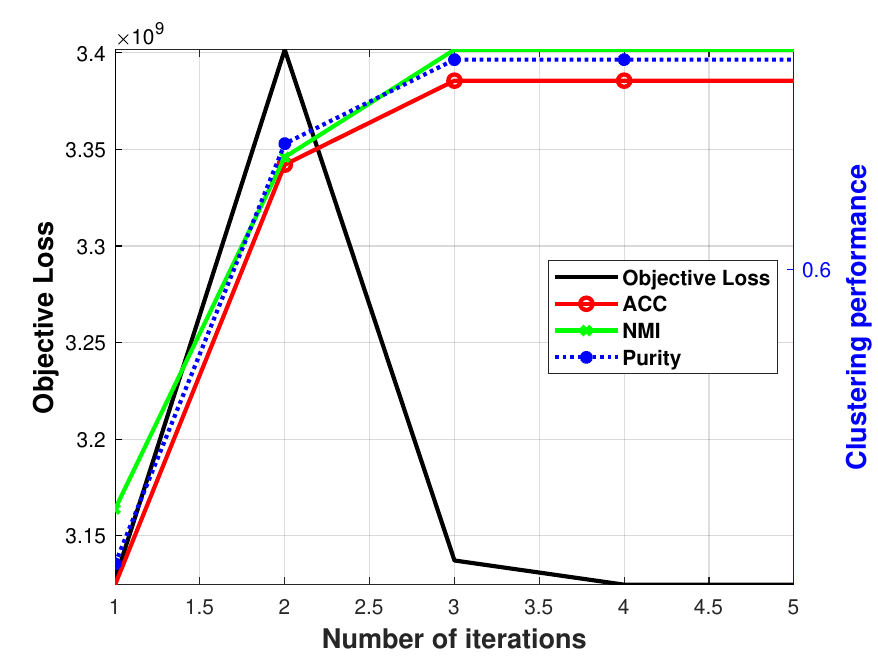}
		\caption{Yaleface}
	\end{subfigure}
	\caption{Clustering performance v.s. objective loss.}
	\label{figcon}
\end{figure}

\section{Conclusion}\label{conclusion}
This paper proposes a self-supervised manifold clustering framework that integrates K-means without centroids and graph embedding. Specifically, we construct the manifold structure through labels to ensure consistency between the labels and the manifold structure. The paper explores the relationship between manifold learning and K-means clustering without centroids, demonstrating that our algorithm is independent of the choice of cluster centers. Furthermore, we prove that maximizing the $\ell_{2,1}$-norm in our algorithm naturally promotes category balance in the clustering process. Compared with existing clustering algorithms that operate in the original space or utilize dimensionality reduction, Our-LPP and Our-MFA show significant performance improvements.

						\ifCLASSOPTIONcompsoc
						\else
						\fi
						\ifCLASSOPTIONcaptionsoff
						\newpage
						\fi
						
						{\small
							\bibliographystyle{IEEEtran}
							\bibliography{aaai25}
						}
						
					\end{document}